\documentclass{article}
\usepackage{journal}
\usepackage{natbib}
\usepackage[mathscr]{eucal}
\usepackage[usenames,dvipsnames]{xcolor}
\usepackage[backref,pageanchor=true,plainpages=false,pdfpagelabels,bookmarks,bookmarksnumbered,breaklinks,pdfborder={0 0 0},colorlinks=true,
	citecolor=Sepia,
	linkcolor=Sepia,
	filecolor=Sepia,
	urlcolor=Sepia
]{hyperref}

\usepackage{framed}
\usepackage{amsmath,amsfonts,amsbsy,amsgen,amscd,mathrsfs,amssymb}
\usepackage[utf8]{inputenc} % allow utf-8 input
\usepackage[T1]{fontenc}    % use 8-bit T1 fonts
\usepackage{hyperref}       % hyperlinks
\usepackage{booktabs}       % professional-quality tables
\usepackage{nicefrac}       % compact symbols for 1/2, etc.
\usepackage{microtype}      % microtypography
\usepackage[framemethod=TikZ]{mdframed}
\mdfsetup{skipabove=\topskip,skipbelow=\topskip}

\usepackage{bm} 
\usepackage{bbm}
\usepackage{xspace}
\usepackage[algo2e,ruled,linesnumbered,vlined]{algorithm2e}
\usepackage{tcolorbox}

\usepackage{amsthm}
\usepackage{mathtools}
\usepackage[shortlabels]{enumitem}

\makeatletter
\def\set@curr@file#1{\def\@curr@file{#1}} %temp workaround for 2019 latex release
\makeatother
\usepackage[load-configurations=version-1]{siunitx} % newer version

\newtheorem{innercustomthm}{Theorem}
\newenvironment{customthm}[1]
  {\renewcommand\theinnercustomthm{#1}\innercustomthm}
  {\endinnercustomthm}
  
\newtheorem{innercustomcor}{Corollary}
\newenvironment{customcor}[1]
  {\renewcommand\theinnercustomcor{#1}\innercustomcor}
  {\endinnercustomcor}

\newtheorem{thm}{Theorem}
\newtheorem{lem}[thm]{Lemma}

\newtheorem{cor}[thm]{Corollary}

\theoremstyle{definition}
\newtheorem{defn}{Definition}
\newtheorem*{defn*}{Definition}
\renewenvironment{proof}[1][]{\par\noindent{\bf Proof #1\ }}{\hfill\BlackBox\\}

\usepackage{prettyref}
\newrefformat{fig}{Figure~\ref{#1}}
\newrefformat{alg}{Algorithm~\ref{#1}}
\newrefformat{def}{Definition~\ref{#1}}
\newrefformat{eqn}{Equation~\ref{#1}}
\newrefformat{cor}{Corollary~\ref{#1}}
\newrefformat{app}{Appendix~\ref{#1}}
\newrefformat{rem}{Remark~\ref{#1}}
\newrefformat{assu}{Assumption~\ref{#1}}
\newrefformat{q}{Question~\ref{#1}}
\newrefformat{clm}{Claim~\ref{#1}}

\jmlrheading{}{2022}{}{}{}{A. Blum, O. Montasser, G. Shakhnarovich \& H. Zhang}
\ShortHeadings{Blum Montasser Shakhnarovich Zhang}{A Theory for Boosting Robustness}
\newcommand{\inbrace}[1]{\left \{ #1 \right \}}
\newcommand{\inparen}[1]{\left ( #1 \right )}
\newcommand{\insquare}[1]{\left [ #1 \right ]}

\newcommand{\abs}[1]{\left\lvert #1 \right\rvert}

\newlength{\dhatheight}

\newcommand{\ceil}[1]{\left \lceil #1 \right \rceil}

\newcommand{\SET}[1]{\inbrace{#1}}

\DeclareMathOperator*{\Ex}{\mathbb{E}}

\DeclareMathOperator*{\Prob}{Pr}

\newcommand{\bbN}{{\mathbb N}}

\newcommand{\bbR}{{\mathbb R}}

\newcommand{\bbA}{{\mathbb A}}
\newcommand{\bbB}{{\mathbb B}}

\let\boldm\bm

\newcommand{\by}{{\boldm y}}

\newcommand{\tby}{\Tilde{{\boldm y}}}

\newcommand{\calC}{\mathcal{C}}
\newcommand{\calD}{\mathcal{D}}

\newcommand{\calU}{\mathcal{U}}

\newcommand{\calX}{\mathcal{X}}
\newcommand{\calY}{\mathcal{Y}}

\newcommand{\B}{\mathrm{B}}

\newcommand{\ind}{\mathbbm{1}}
\newcommand{\Risk}{{\rm R}}

\newcommand{\MAJ}{{\rm MAJ}}

\newcommand{\Rob}{{\rm Rob}}
\newcommand{\CAS}{{\rm CAS}}
\newcommand{\wrt}{with respect to }

\newcommand{\removed}[1]{}

\newcommand{\RoBoost}{\hyperref[alg:cascade]{\textup{$\beta$-RoBoost}}}
\newcommand{\URoBoost}{\hyperref[alg:unlabeled]{\textup{$\beta$-URoBoost}}}

\newcommand{\alphaBoost}{\hyperref[alg:alphaboost]{\textup{$\alpha$-Boost}}}
\usepackage{times}
\begin{document}

\title{Boosting Barely Robust Learners:\\
A New Perspective on Adversarial Robustness}
\author{%
 \name{Avrim Blum} \email{avrim@ttic.edu}\\
 \name{Omar Montasser} \email{omar@ttic.edu}\\
 \name{Greg Shakhnarovich} \email{greg@ttic.edu}\\
 \addr Toyota Technological Institute at Chicago\\
 \name{Hongyang Zhang} \email{hongyang.zhang@uwaterloo.ca}\\
 \addr University of Waterloo
}

\maketitle

\begin{abstract}%
We present an oracle-efficient algorithm for boosting the \emph{adversarial robustness} of \emph{barely robust} learners. Barely robust learning algorithms learn predictors that are adversarially robust only on a small fraction $\beta \ll 1$ of the data distribution. Our proposed notion of barely robust learning requires robustness with respect to a ``larger'' perturbation set; which we show is \emph{necessary} for strongly robust learning, and that weaker relaxations are \emph{not} sufficient for strongly robust learning. Our results reveal a qualitative and quantitative {\em equivalence} between two seemingly unrelated problems: strongly robust learning and barely robust learning.
\end{abstract}

\begin{keywords}%
Adversarially Robust Learning, Boosting
\end{keywords}

\section{Introduction}
\label{sec:intro}
We consider the problem of learning predictors that are {\em robust} to adversarial examples at test time. That is, we would like to be robust against a perturbation set $\calU:\calX \to 2^{\calX}$, where $\calU(x)\subseteq \calX$ is the set of allowed perturbations that an adversary might replace $x$ with, e.g., $\calU$ could be perturbations of bounded $\ell_p$-norms \citep{DBLP:journals/corr/GoodfellowSS14}. The goal is to learn a predictor $h$ with small {\em robust risk}:
\begin{equation}
    \label{eqn:rob-risk}
\Risk_{\calU}(h;\calD) \triangleq \Prob_{(x,y)\sim \calD}\insquare{\exists z\in \calU(x): h(z)\neq y}.
\end{equation}

Adversarially robust learning has proven to be quite challenging in practice, where current adversarial learning methods typically learn predictors with low natural error but robust only on a small fraction of the data. For example, according to the \textup{RobustBench} leaderboard \citep{croce2020robustbench}, the highest achieved robust accuracy \wrt $\ell_\infty$ perturbations on \textup{CIFAR10} is $\approx 66\%$ and on \textup{ImageNet} is $\approx 38\%$. Can we leverage existing methods and go beyond their limits? This motivates us to pursue the idea of \emph{boosting} robustness, and study the following theoretical question:
\begin{center}
    \textit{Can we boost {\em barely} robust learning algorithms to learn predictors with high {\em robust} accuracy?}
\end{center}
That is, given a \emph{barely} robust learning algorithm $\bbA$ which can only learn predictors robust on say $\beta=10\%$ fraction of the data distribution, we are asking whether it is possible to \emph{boost} the robustness of $\bbA$ and learn predictors with high \emph{robust} accuracy, say $90\%$. We want to emphasize that we are interested here in extreme situations when the robustness parameter $\beta \ll 1$. We are interested in generic boosting algorithms that take as input a black-box learner $\bbA$ and a specification of the perturbation set $\calU$, and output a predictor with high robust accuracy by repeatedly calling $\bbA$.

In this work, by studying the question above, we offer a new perspective on adversarial robustness. Specifically, we discover a qualitative and quantitative {\em equivalence} between two seemingly unrelated problems: strongly robust learning and barely robust learning. We show that barely robust learning implies strongly robust learning through a novel algorithm for \emph{boosting} robustness. As we elaborate below, our proposed notion of barely robust learning requires robustness \wrt a ``larger'' perturbation set. We also show that this is \emph{necessary} for strongly robust learning, and that weaker relaxations of barely robust learning do not imply strongly robust learning. 

\subsection{Main Contributions}

When formally studying the problem of boosting robustness, an important question emerges which is: what notion of ``barely robust'' learning is required for boosting robustness? As we shall show, this is not immediately obvious. One of the main contributions of this work is the following key definition of \emph{barely robust} learners:
\begin{defn} [Barely Robust Learner]
\label{def:barelyrobust}
Learner $\bbA$ $(\beta, \epsilon, \delta)$-barely-robustly-learns a concept $c:\calX\to\calY$ w.r.t. $\calU^{-1}(\calU)$ if $\exists m_{\bbA}(\beta,\epsilon,\delta)\in \bbN$ such that for any distribution $D$ over $\calX$ s.t.~$\Prob_{x\sim D}\insquare{\exists z\in \calU(x): c(z)\neq c(x)}=0$, with probability at least $1-\delta$ over $S=\SET{(x_i,c(x_i))}_{i=1}^{m}\sim D_c$, $\bbA$ outputs a predictor $\hat{h}=\bbA(S)$ satisfying:
\[{\Prob_{x\sim D}\insquare{\forall \Tilde{x}\in \calU^{-1}(\calU)(x): \hat{h}(\Tilde{x})= \hat{h}(x)}\geq \beta}~~~\text{and}~~~{\Prob_{x\sim D}\insquare{\hat{h}(x)\neq c(x)}\leq \epsilon}.\]
\end{defn}
Notice that we require $\beta$-robustness \wrt a ``larger'' perturbation set $\calU^{-1}(\calU)$. Specifically, $\calU^{-1}(\calU)(x)$ is the set of all \emph{natural} examples $\Tilde{x}$ that share an adversarial perturbation $z$ with $x$ (see \prettyref{eqn:youinverseyou}). E.g., if $\calU(x)$ is an $\ell_p$-ball with radius $\gamma$, then $\calU^{-1}(\calU)(x)$ is an $\ell_p$-ball with radius $2\gamma$. 

On the other hand, $(\epsilon,\delta)$-robustly-learning a concept $c$ \wrt $\calU$ is concerned with learning a predictor $\hat{h}$ from samples $S$ with small robust risk $\Risk_\calU(\hat{h}; D_c)\leq \epsilon$ with probability at least $1-\delta$ over $S\sim D^m_c$ (see \prettyref{eqn:rob-risk} and \prettyref{def:strongly-robust}), where we are interested in robustness \wrt $\calU$ and {\em not} $\calU^{-1}(\calU)$. Despite this qualitative difference between \emph{barely robust} learning and \emph{strongly robust} learning, we provably show next that they are in fact \emph{equivalent}. 

Our main algorithmic result is \RoBoost, an oracle-efficient boosting algorithm that boosts barely robust learners to strongly robust learners:
\begin{customthm}{1}
For any perturbation set $\calU$, \RoBoost\ $(\epsilon,\delta)$-robustly-learns any target concept $c:\calX\to\calY$  w.r.t.~$\calU$ using $T= \frac{\ln(2/\epsilon)}{\beta}$ black-box oracle calls to any $(\beta,\frac{\beta\epsilon}{2},\frac{\delta}{2T})$-barely-robust learner $\bbA$ for $c$ w.r.t.~$\calU^{-1}(\calU)$, with sample complexity $\frac{4Tm_{\bbA}}{\epsilon}$, where $m_{\bbA}$ is the sample complexity of learner $\bbA$.
\end{customthm}

The result above shows that barely robust learning is \emph{sufficient} for strongly robust learning. An important question remains, however: is our proposed notion of barely robust learning \emph{necessary} for strongly robust learning? In particular, our proposed notion of barely robust learning requires $\beta$-robustness \wrt a ``larger'' perturbation set $\calU^{-1}(\calU)$, instead of the actual perturbation set $\calU$ that we care about. We provably show next that this is \emph{necessary}. 

\begin{customthm}{6}
For any $\calU$, learner $\bbB$, and $\epsilon\in(0,1/4)$, if $\bbB$ $(\epsilon,\delta)$-robustly-learns some unknown target concept $c$ w.r.t.~$\calU$, then there is a learner $\Tilde{\bbB}$ that $(\frac{1-\epsilon}{2}, 2\epsilon, 2\delta)$-barely-robustly-learns $c$ w.r.t. $\calU^{-1}(\calU)$.
\end{customthm}

This still does \emph{not} rule out the possibility that boosting robustness is possible even with the weaker requirement of $\beta$-robustness with respect to $\calU$. But we show next that, indeed, barely robust learning \wrt $\calU$ is \emph{not sufficient} for strongly robust learning \wrt $\calU$:

\begin{customthm}{8} 
There is a space $\calX$, a perturbation set $\calU$, and a class of concepts $\calC$ s.t. $\calC$ is $(\beta=\frac{1}{2},\epsilon=0,\delta)$-barely-robustly-learnable w.r.t~$\calU$, but $\calC$ is \emph{not} $(\epsilon,\delta)$-robustly-learnable w.r.t.~$\calU$ for any $\epsilon<1/2$.
\end{customthm}

Our results offer a new perspective on adversarially robust learning. We show that two seemingly unrelated problems: barely robust learning w.r.t. $\calU^{-1}(\calU)$ and strongly robust learning w.r.t. $\calU$, are in fact \emph{equivalent}. The following corollary follows from \prettyref{thm:boost-realizable} and \prettyref{thm:stronglytobarely}.
\begin{customcor}{I}
For any class $\calC$ and any perturbation set $\calU$, $\calC$ is {\em strongly} robustly learnable \wrt $\calU$ if and only if $\calC$ is {\em barely} robustly learnable \wrt $\calU^{-1}(\calU)$.
\end{customcor}

We would like to note that in our treatment of boosting robustness, having a separate robustness parameter $\beta$ and a natural error parameter $\epsilon$ allows us to consider regimes where $\beta < \frac{1}{2}$ and $\epsilon$ is small. This models typical scenarios in practice where learning algorithms are able to learn predictors with reasonably low natural error but the predictors are only barely robust. More generally, this allows us to explore the relationship between $\beta$ and $\epsilon$ in terms of boosting robustness (see \prettyref{sec:dis} for a more elaborate discussion).

\paragraph{Landscape of Boosting Robustness.} Our results reveal an interesting landscape for boosting robustness when put in context of prior work. When the robustness parameter $\beta > \frac{1}{2}$, it is known from prior work that $\beta$-robustness \wrt $\calU$ \emph{suffices} for boosting robustness \citep[see e.g.,][]{pmlr-v99-montasser19a,DBLP:journals/corr/abs-2103-01276}, which is witnessed by the \alphaBoost~algorithm \citep{schapire:12}. When the robustness parameter $\beta \leq \frac{1}{2}$, our results show that boosting is still \emph{possible}, but $\beta$-robustness \wrt $\calU^{-1}(\calU)$ is \emph{necessary} and we {\em cannot} boost robustness with $\beta$-robustness \wrt $\calU$. 

In fact, by combining our algorithm \RoBoost~with \alphaBoost, we obtain an even \emph{stronger} boosting result that only requires barely robust learners with a natural error parameter that does \emph{not} scale with the targeted robust error. Beyond that, our results imply that we can even boost robustness \wrt $\calU^{-1}(\calU)$. This is summarized in the following corollary which follows from \prettyref{thm:boost-realizable}, \prettyref{lem:weak-robust-learner}, and \prettyref{thm:stronglytobarely}.

\begin{customcor}{II}[Landscape of Boosting Robustness]
\label{cor:summary}
Let $\calC$ be a class of concepts. For fixed $\epsilon_0,\delta_0=(\frac13,\frac13)$ and any target $\epsilon < \epsilon_0$ and $\delta > \delta_0$:
\begin{enumerate}
    \item If $\calC$ is $(\beta,\frac{\beta\epsilon_0}{2},\frac{\beta\delta_0}{\ln(2/\epsilon_0)})$-barely-robustly-learnable w.r.t.~$\calU^{-1}(\calU)$, then $\calC$ is $(\epsilon_0,\delta_0)$-robustly-learnable w.r.t.~$\calU$.
    \item If $\calC$ is $(\epsilon_0,\delta_0)$-robustly-learnable w.r.t.~$\calU$, then $\calC$ is $(\epsilon,\delta)$-robustly learnable w.r.t.~$\calU$.
    \item If $\calC$ is $(\epsilon,\delta)$-robustly learnable w.r.t.~$\calU$, then $\calC$ is $(\frac{1-\epsilon}{2},2\epsilon,2\delta)$-barely-robustly-learnable w.r.t.~$\calU^{-1}(\calU)$.
\end{enumerate}
In particular, $1\Rightarrow 2 \Rightarrow 3$ reveals that we can also algorithmically boost robustness w.r.t.~$\calU^{-1}(\calU)$.
\end{customcor}

\paragraph{Extra Features.} In \prettyref{thm:boost-unlabeled}, we show that a variant of our boosting algorithm, \URoBoost, can boost robustness using \emph{unlabeled} data when having access to a barely robust learner $\bbA$ that is tolerant to small noise in the labels. In \prettyref{app:granular}, we discuss an idea of obtaining robustness at different levels of granularity through our boosting algorithm. Specifically, when $\calU(x)$ is a metric-ball around $x$ with radius $\gamma$, we can learn a predictor $\hat{h}$ with different robustness levels: $\gamma,\frac{\gamma}{2},\frac{\gamma}{4}, \dots$, in different regions of the distribution.

\subsection{Related Work} To put our work in context, the classic and pioneering works of \citep{k-thb-88,DBLP:journals/ml/Schapire90,DBLP:conf/colt/Freund90,DBLP:journals/jcss/FreundS97} explored the question of boosting the \emph{accuracy} of \emph{weak} learning algorithms, from accuracy slightly better than $\frac{1}{2}$ to arbitrarily high accuracy in the realizable PAC learning setting. Later works have explored boosting the accuracy in the agnostic PAC setting \citep[see e.g.,][]{DBLP:conf/nips/KalaiK09}. In this work, we are interested in the problem of boosting \emph{robustness} rather than accuracy. In particular, boosting robustness of learners $\bbA$ that are highly accurate on \emph{natural} examples drawn from the data distribution, but robust only on some $\beta$ fraction of data distribution. We consider this problem in the \emph{robust realizable} setting, that is, when the unknown target concept $c$ has zero robust risk $\Risk_\calU(c;D_c)=0$. 

We already know from prior work \citep[see e.g.,][]{pmlr-v99-montasser19a,DBLP:journals/corr/abs-2103-01276} that if the robustness parameter $\beta>\frac{1}{2}$, then barely robust learning \wrt $\calU$ implies strongly robust learning \wrt $\calU$. Essentially, in this case, we have \emph{weak} learners \wrt the robust risk $\Risk_\calU$, and the original boosting algorithms such as the $\alpha$-Boost algorithm \cite[Section 6.4.2]{schapire:12} can boost the robust risk. In this work, we focus on boosting \emph{barely} robust learners, i.e., mainly when the robustness parameter $\beta < 1/2$, but in general our algorithm works for any $0< \beta \leq 1$. 

\citet{pmlr-v99-montasser19a} studied the problem of adversarially robust learning (as in \prettyref{def:strongly-robust}). They showed that if a hypothesis class $\calC$ is PAC learnable non-robustly (i.e., $\calC$ has finite VC dimension), then $\calC$ is adversarially robustly learnable. This result, however, is not constructive and the robust learning algorithm given does not directly use a black-box non-robust learner. Later on, \citet{DBLP:conf/nips/MontasserHS20} studied a more constructive version of the same question: reducing strongly robust learning to \emph{non-robust} PAC learning when given access to black-box \emph{non-robust} PAC learners. This is different from the question we study in this work. In particular, we explore the relationship between strongly robust learning and barely robust learning, and we present a boosting algorithm for learners that already have some non-trivial robustness guarantee $\beta>0$.

\section{Preliminaries}
\label{sec:setup}
Let $\calX$ denote the instance space and $\calY$ denote the label space. We would like to be robust \wrt a perturbation set $\calU:\calX \to 2^{\calX}$, where $\calU(x)\subseteq \calX$ is the set of allowed adversarial perturbations that an adversary might replace $x$ with at test time. Denote by $\calU^{-1}$ the inverse image of $\calU$, where for each $z\in \calX$, $\calU^{-1}(z)=\SET{x\in\calX: z\in\calU(x)}$. Observe that for any $x,z\in\calX$ it holds that $z\in\calU(x) \Leftrightarrow x\in \calU^{-1}(z)$. Furthermore, when $\calU$ is symmetric, where for any $x,z\in \calX, z\in \calU(x)\Leftrightarrow x\in \calU(z)$, it holds that $\calU=\calU^{-1}$. For each $x\in\calX$, denote by $\calU^{-1}(\calU)(x)$ the set of all \emph{natural} examples $\Tilde{x}$ that share some adversarial perturbation $z$ with $x$, i.e., 
\begin{equation}
\label{eqn:youinverseyou}
    \calU^{-1}(\calU)(x) = \cup_{z\in \calU(x)} \calU^{-1}(z)=\SET{\Tilde{x}:\exists z\in \calU(x) \cap \calU(\Tilde{x})}.
\end{equation}
For example, when $\calU(x)=\B_\gamma(x)= \SET{z \in \calX: \rho(x,z)\leq \gamma}$ where $\gamma>0$ and $\rho$ is some metric on $\calX$ (e.g., $\ell_p$-balls), then $\calU^{-1}(\calU)(x)=\B_{2\gamma}(x)$.

For any classifier $h:\calX\to \calY$ and any $\calU$, denote by $\Rob_\calU(h)$ the \emph{robust region} of $h$ with respect to $\calU$ defined as:
\begin{equation}
\label{eqn:robustregion}
\Rob_\calU(h)\triangleq \SET{ x\in \calX: \forall z\in \calU(x), h(z)=h(x)}.
\end{equation}

\begin{defn} [Strongly Robust Learner]
\label{def:strongly-robust}
Learner $\bbB$ $(\epsilon, \delta)$-robustly-learns a concept $c:\calX\to\calY$ with respect to $\calU$ if there exists $ m(\epsilon,\delta)\in \bbN$ s.t.~for any distribution $D$ over $\calX$ satisfying $\Prob_{x\sim D}\insquare{\exists z\in \calU(x): c(z)\neq c(x)}=0$, with probability at least $1-\delta$ over $S=\SET{(x_i,c(x_i))}_{i=1}^{m}\sim D_c$, $\bbB$ outputs a predictor $\hat{h}=\bbB(S)$ satisfying:
\[\Risk_\calU(\hat{h};D_c)=\Prob_{x\sim D} \insquare{\exists z\in \calU(x): \hat{h}(z)\neq c(x)}\leq \epsilon.\]
\end{defn}

\section{Boosting a Barely Robust Learner to a Strongly Robust Learner}
\label{sec:forward-direction}

We present our main result in this section: \RoBoost\ is an algorithm for \emph{boosting} the \emph{robustness} of \emph{barely} robust learners. Specifically, in \prettyref{thm:boost-realizable}, we show that given a \emph{barely} robust learner $\bbA$ for some unknown target concept $c$ (according to \prettyref{def:barelyrobust}), it is \emph{possible} to strongly robustly learn $c$ with \RoBoost\ by making black-box oracle calls to $\bbA$. 

\begin{algorithm2e}[H]
\caption{\textup{$\beta$-RoBoost} --- Boosting {\em barely} robust learners.}
\label{alg:cascade}
%\label{cas}
\SetKwInput{KwInput}{Input}                % Set the Input
\SetKwInput{KwOutput}{Output}              % set the Output
\SetKwFunction{pred}{CAS}
\SetKwFunction{rej}{RejectionSampling}
\SetKwFunction{sel}{SEL}
\DontPrintSemicolon
  \KwInput{Sampling oracle for distribution $D_c$, black-box $(\beta,\epsilon,\delta)$-barely-robust learner $\bbA$.}
Set $T=\frac{\ln(2/\epsilon)}{\beta}$, and $m=\max\SET{m_\bbA(\beta, \frac{\beta\epsilon}{2},\frac{\delta}{2T}), 4 \ln\inparen{\frac{2T}{\delta}}}$.\;
\While{$1 \leq t\leq T$}{
    Call \rej on $h_1,\dots, h_{t-1}$ and $m$, and let $\Tilde{S}_t$ be the returned dataset.\;
    \textbf{If $\Tilde{S}_t\neq \emptyset$, then} call learner $\bbA$ on $\Tilde{S}_t$ and let $h_t=\bbA(\Tilde{S}_t)$ be its output. \textbf{Otherwise}, break.
}

\KwOutput{The cascade predictor defined as $\CAS(h_1,\dots,h_T)(z)\triangleq G_{h_{s}}(z) \text{ where } s=\min\SET{1\leq t\leq T: G_{h_{t}}(z)\neq \perp}$, and \emph{selective classifiers}~~~
$G_{h_t}(z) \triangleq \begin{cases} 
                                                              y,& \text{if } \inparen{\exists y\in\calY}\inparen{\forall{\Tilde{x}\in\calU^{-1}(z)}}: h(\Tilde{x})=y;\\
                                                              \perp, &\text{otherwise.}
                                                           \end{cases}$.}
\BlankLine
  \SetKwProg{Fn}{}{:}{\KwRet}
  \Fn{\rej{predictors $h_1,\dots, h_t$, and sample size $m$}}{
    \For{$1\leq i\leq m$}{
    Draw samples $(x,y)\sim D$ until sampling an $(x_i,y_i)$ s.t.: $\forall_{t'\leq t}\exists_{z \in \calU(x_i)} G_{h_{t'}}(z)=\perp$.\;{\scriptsize\tcp*{sampling from the region of $D$ where \emph{all} predictors $h_1,\dots, h_{t}$ are \emph{not} robust.}}
    If this costs more than $\frac{4}{\epsilon}$ samples from $D$, abort and return an empty dataset $\Tilde{S}=\emptyset$.\;{\scriptsize\tcp*{If the mass of the non-robust region is small, then we can safely terminate.}}
  }
  Output dataset $\Tilde{S}=\SET{(x_1,y_1),\dots,(x_m,y_m)}$.\;
 }
\end{algorithm2e}
\begin{thm}
\label{thm:boost-realizable}
For any perturbation set $\calU$, \RoBoost\ $(\epsilon,\delta)$-robustly-learns any target concept $c$ w.r.t.~$\calU$ using $T= \frac{\ln(2/\epsilon)}{\beta}$ black-box oracle calls to any $(\beta,\frac{\beta\epsilon}{2},\frac{\delta}{2T})$-barely-robust learner $\bbA$ for $c$ w.r.t.~$\calU^{-1}(\calU)$, with total sample complexity
\[m(\epsilon,\delta)\leq \frac{4T\max\SET{m_{\bbA}(\beta,\frac{\beta\epsilon}{2},\frac{\delta}{2T}), 4 \ln\inparen{\frac{2T}{\delta}}}}{\epsilon}.\]
\end{thm}

In fact, we present next an even stronger result for boosting $(\beta,\epsilon_0,\delta_0)$-barely-robust-learners with fixed error $\epsilon_0=\frac{\beta}{6}$ and confidence $\delta_0=\frac{\beta}{6\ln(6)}$. This is established by combining two boosting algorithms: \RoBoost\ from \prettyref{thm:boost-realizable} and \alphaBoost~from earlier work \citep[see e.g.,][]{pmlr-v99-montasser19a,schapire:12} which is presented in \prettyref{app:stronger-boosting} for convenience. The main idea is to perform two layers of boosting. In the first layer, we use \RoBoost~to get a $(\frac{1}{3},\frac{1}{3})$-robust-learner w.r.t.~$\calU$ from a $(\beta,\epsilon_0,\delta_0)$-barely-robust-learner $\bbA$ w.r.t.~$\calU^{-1}(\calU)$. Then, in the second layer, we use \alphaBoost~to boost \RoBoost~from a $(\frac{1}{3},\frac{1}{3})$-robust-learner to an $(\epsilon,\delta)$-robust-learner w.r.t.~$\calU$.

\begin{cor}
\label{cor:stronger-boosting}
For any perturbation set $\calU$, \alphaBoost~combined with \RoBoost~$(\epsilon,\delta)$-robustly-learn any target concept $c$ w.r.t.~$\calU$ using $T = O(\log(m)\inparen{\log(1/\delta)+\log\log m})\cdot \frac{1}{\beta}$ black-box oracle calls to any $(\beta,\frac{\beta}{6},\frac{\beta}{6\ln(6)})$-barely-robust-learner $\bbA$ for $c$ w.r.t.~$\calU^{-1}(\calU)$, with total sample complexity
\[m(\epsilon,\delta)=O\inparen{\frac{m_0}{\beta\epsilon}\log^2\inparen{\frac{m_0}{\beta\epsilon}}+\frac{\log(1/\delta)}{\epsilon}}, \text{ where }m_0=\max\SET{m_{\bbA}\inparen{\beta,\frac{\beta}{6},\frac{\beta}{6\ln(6)}},4\ln\inparen{\frac{6\ln(6)}{\beta}}}.\]
\end{cor}

We begin with describing the intuition behind \RoBoost, and then we will prove \prettyref{thm:boost-realizable} and \prettyref{cor:stronger-boosting}.
\paragraph{High-level Strategy.} Let $D_c$ be the unknown distribution we want to robustly learn. Since $\bbA$ is a \emph{barely} robust learner for $c$, calling learner $\bbA$ on an i.i.d. sample $S$ from $D_c$ will return a predictor $h_1$, where $h_1$ is robust only on a region $R_1\subseteq \calX$ of small mass $\beta>0$ under distribution $D$, $\Prob_{x\sim D}[x\in R_1]\geq \beta$. We can trust the predictions of $h_1$ in the region $R_1$, but not in the complement region $\bar{R}_1$ where it is not robust. For this reason, we will use a \emph{selective} classifier $G_{h_1}$ (see \prettyref{eqn:selective-classifier}) which makes predictions on \emph{all} adversarial perturbations in region $R_1$, but \emph{abstains} on adversarial perturbations not from $R_1$. In each round $t>1$, the strategy is to focus on the region of distribution $D$ where \emph{all} predictors $h_1,\dots, h_{t-1}$ returned by $\bbA$ so far are \emph{not} robust. By \emph{rejection sampling}, \RoBoost\ gives barely robust learner $\bbA$ a sample $\Tilde{S}_t$ from this non-robust region, and then $\bbA$ returns a predictor $h_t$ with robustness at least $\beta$ in this region. Thus, in each round, we shrink by a factor of $\beta$ the mass of region $D$ where the predictors learned so far are not robust. After $T$ rounds, \RoBoost\ outputs a \emph{cascade} of \emph{selective} classifiers $G_{h_1}, \dots, G_{h_T}$ where roughly each selective classifier $G_{h_t}$ is responsible for making predictions in the region where $h_t$ is robust.

As mentioned above, one of the main components in our boosting algorithm is \emph{selective classifiers} that essentially abstain from predicting in the region where they are {\em not} robust. Formally, for any classifier $h:\calX\to\calY$ and any $\calU$, denote by $G_h:\calX \to \calY \cup \SET{\perp}$ a \emph{selective} classifier defined as:
\begin{equation}
\label{eqn:selective-classifier}
G_h(z) \triangleq \begin{cases} 
                                                              y,& \text{if } \inparen{\exists y\in\calY}\inparen{\forall{\Tilde{x}\in\calU^{-1}(z)}}: h(\Tilde{x})=y;\\
                                                              \perp, &\text{otherwise.}
                                                           \end{cases}
\end{equation}  
Before proceeding with the proof of \prettyref{thm:boost-realizable}, we prove the following key Lemma about \emph{selective classifier} $G_h$ which states that $G_h$ will {\em not} abstain in the region where $h$ is robust, and whenever $G_h$ predicts a label for a perturbation $z$ then we are guaranteed that this is the same label that $h$ predicts on the corresponding natural example $x$ where $z\in\calU(x)$. 

\begin{lem}
\label{lem:selective-guarantee}
For any distribution $D$ over $\calX$, any $\calU$, and any $\beta>0$, given a classifier $h:\calX\to\calY$ satisfying $\Prob_{x\sim D}\insquare{x\in \Rob_{\calU^{-1}(\calU)}(h)}\geq \beta$, then the \emph{selective} classifier $G_h:\calX\to\calY\cup\SET{\perp}$ (see \prettyref{eqn:selective-classifier}) satisfies:
\[\Prob_{x\sim D}\insquare{\forall z\in \calU(x): G_h(z)=h(x)}\geq \beta~\text{and}~\Prob_{x\sim D}\insquare{\forall z\in\calU(x): G_h(z)=h(x) \vee G_h(z)=\perp}=1.\]
\end{lem}

\begin{proof}
Observe that, by the definition of the robust region of $h$ w.r.t. $\calU^{-1}(\calU)$ (see \prettyref{eqn:robustregion}), for any $x\in\Rob_{\calU^{-1}(\calU)}(h)$ the following holds:$\inparen{\forall z\in \calU(x)}\inparen{\forall\Tilde{x}\in \calU^{-1}(z)}: h(\Tilde{x})=h(x)$. By the definition of the \emph{selective} classifier $G_h$ (see \prettyref{eqn:selective-classifier}), this implies that:~$\forall z\in \calU(x), G_h(z) = h(x)$.

Since $\Prob_{x\sim D}\insquare{x\in \Rob_{\calU^{-1}(\calU)}(h)}\geq \beta$, the above implies $\Prob_{x\sim D}\insquare{\forall z\in \calU(x): G_h(z)=h(x)}\geq \beta$. 
Furthermore, for any $x\in \calX$ and any $z\in \calU(x)$, by definition of $\calU^{-1}$, $x\in\calU^{-1}(z)$. Thus, by definition of $G_h$ (see \prettyref{eqn:selective-classifier}), if $G_h(z)=y$ for some $y\in \calY$, then it holds that $h(x)=y$. Combined with the above, this implies that
\[
    \Prob_{x\sim D}\insquare{\forall z\in\calU(x): G_h(z)=h(x) \vee G_h(z)=\perp}=1.
\]
\end{proof}

We are now ready to proceed with the proof of \prettyref{thm:boost-realizable}.

\begin{proof}[of \prettyref{thm:boost-realizable}.]
Let $\calU$ be an arbitrary adversary, and $\bbA$ a $(\beta, \epsilon,\delta)$-barely-robust learner for some \emph{unknown} target concept $c:\calX\to\calY$ with respect to $\calU^{-1}(\calU)$. We will show that \RoBoost\  $(\epsilon,\delta)$-robustly-learns $c$ with respect to $\calU$. Let $D$ be some unknown distribution over $\calX$ such that $\Prob_{x\sim D}\insquare{\exists z\in \calU(x): c(z)\neq c(x)}=0$. Let $\epsilon>0$ be our target robust error, and $\epsilon'$ be the error guarantee of learner $\bbA$ (which we will set later to be $\frac{\beta\epsilon}{2}$).

Without loss of generality, suppose that \RoBoost\ ran for $T=\frac{\ln(2/\epsilon)}{\beta}$ rounds (Step 5 takes care of the scenario where progress is made faster). Let $h_1=\bbA(\Tilde{S}_1),\dots, h_{T}=\bbA(\Tilde{S}_T)$ be the predictors returned by learner $\bbA$ on rounds $1\leq t \leq T$. For any $1\leq t \leq T$ and any $x\in \calX$, denote by ${R}_t$ the event that $x \in \Rob_{\calU^{-1}(\calU)}(h_t)$, and by $\Bar{R}_t$ the event that $x \notin \Rob_{\calU^{-1}(\calU)}(h_t)$ (as defined in \prettyref{eqn:robustregion}). Observe that by properties of $\bbA$ (see \prettyref{def:barelyrobust}), we are guaranteed that
\begin{equation}
\label{eqn:barelyrobustguarantee}
    \forall 1 \leq t \leq T: \Prob_{x\sim D_t}\insquare{R_t}\geq \beta \text{ and } \Prob_{x\sim D_t}\insquare{h_t(x)\neq c(x)}\leq \epsilon',
\end{equation}
where $D_t$ is the distribution from which $\Tilde{S}_t$ is drawn. In words, distribution $D_t$ is a conditional distribution focusing on the region of distribution $D$ where all predictors $h_1,\dots, h_{t-1}$ are \emph{non-robust}. Specifically, for any $x\sim D$, in case $x\in R_1 \cup \dots \cup R_{t-1}$ (i.e., there is a predictor among $h_1,\dots, h_{t-1}$ that is robust on $x$), then \prettyref{lem:selective-guarantee} guarantees that one of the selective classifiers $G_{h_1},\dots, G_{h_{t-1}}$ will not abstain on any $z\in \calU(x)$: $\exists t'\leq t-1, \forall z\in \calU(x), G_{h_{t'}}(z)\neq \perp$. In case $x \in \Bar{R}_1 \cap \cdots \cap \bar{R}_{t-1}$, then, by definition of $\Bar{R}_1,\dots, \Bar{R}_{t-1}$, each of the selective classifiers $G_{h_1},\dots, G_{h_{t-1}}$ can be forced to abstain: $\forall t'\leq t-1, \exists z\in \calU(x), G_{h_{t'}}(z)=\perp$. Thus, in Step 7 of \RoBoost, \emph{rejection sampling} guarantees that $\Tilde{S}_t$ is a sample drawn from distribution $D$ conditioned on the region $\Bar{R}_{1:t-1}\triangleq\Bar{R}_1 \cap \dots \cap\Bar{R}_{t-1}$. Formally, distribution $D_t$ is defined such that for any measurable event $E$:
\begin{equation}
\label{eqn:cond-dist}
    \Prob_{x\sim D_t}\insquare{E}\triangleq \Prob_{x\sim D} \insquare{E | \Bar{R}_1 \cap \dots \cap\Bar{R}_{t-1}}.
\end{equation}
\paragraph{Low Error on Natural Examples.} \prettyref{lem:selective-guarantee} guarantees that whenever any of the \emph{selective} classifiers $G_{h_1},\dots,G_{h_{T}}$ (see \prettyref{eqn:selective-classifier}) chooses to classify an instance $z$ at test-time their prediction will be correct with high probability. We consider two cases. First, in case of event ${\rm R}_t$, $x\in \Rob_{\calU^{-1}(\calU)}(h_t)$, and therefore, by \prettyref{lem:selective-guarantee}, $\forall z\in \calU(x): G_{h_t}(z)=h_{t}(x)$. Thus, \prettyref{eqn:cond-dist} implies that $\forall 1 \leq t \leq T$:
\begin{equation}
\begin{split}
\label{eqn:naterr-robust}
     &\quad\Prob_{x\sim D} \insquare{ {R}_t \wedge \inparen{\exists{z \in \calU(x)}: (G_{h_t}(z)\neq \perp) \wedge (G_{h_t}(z) \neq c(x))} |\bar{R}_{1:t-1}} \\
     &= \Prob_{x\sim D_t} \insquare{ {R}_t \wedge \inparen{\exists{z \in \calU(x)}: (G_{h_t}(z)\neq \perp) \wedge (G_{h_t}(z) \neq c(x))}}\\
    &= \Prob_{x\sim D_t} \insquare{ {R}_t \wedge (h_t(x) \neq c(x)) }.
\end{split}
\end{equation}

Second, in case of the complement event $\Bar{R}_t$, $x\notin \Rob_{\calU^{-1}(\calU)}(h_t)$. Therefore, by \prettyref{lem:selective-guarantee}, $\forall z\in \calU(x)$, we have $G_{h_t}(z) = \perp \text{ or } G_{h_t}(z) = h_t(x)$. Thus,
\begin{equation}
\label{eqn:naterr-nonrobust}
\begin{split}
    &\quad \Prob_{x\sim D} \insquare{ \Bar{R}_t \wedge \inparen{\exists{z \in \calU(x)}: (G_{h_t}(z)\neq \perp) \wedge (G_{h_t}(z) \neq c(x))} |\bar{R}_{1:t-1}} \\
    &= \Prob_{x\sim D_t} \insquare{\Bar{R}_t \wedge \inparen{\exists{z \in \calU(x)}: (G_{h_t}(z)\neq \perp) \wedge (G_{h_t}(z) \neq c(x))}}\leq \Prob_{x\sim D_t} \insquare{\Bar{R}_t \wedge (h_t(x) \neq c(x))}.
\end{split}
\end{equation}

By law of total probability \prettyref{eqn:naterr-robust}, \prettyref{eqn:naterr-nonrobust}, and \prettyref{eqn:barelyrobustguarantee}, 
\begin{equation}
\label{eqn:naterr-total}
\begin{split}
    \Prob_{x\sim D} &\insquare{ ({R}_t \vee \bar{R}_t) \wedge \inparen{\exists{z \in \calU(x)}: (G_{h_t}(z)\neq \perp) \wedge (G_{h_t}(z) \neq c(x))} |\bar{R}_{1:t-1}}\\
    &\leq \Prob_{x\sim D_t} \insquare{R_t \wedge h_t(x) \neq c(x)} + \Prob_{x\sim D_t} \insquare{\Bar{R}_t \wedge h_t(x) \neq c(x)} = \Prob_{x\sim D_t} \insquare{h_t(x) \neq c(x)} \leq \epsilon'.
\end{split}
\end{equation}

\paragraph{Boosted Robustness.} We claim that for each $1 \leq t \leq T: \Prob_{x\sim D}\insquare{\Bar{R}_{1:t}} \leq (1-\beta)^t$.  We proceed by induction on the number of rounds $1\leq t\leq T$.
In the base case, when $t=1$, $D_1=D$ and by \prettyref{eqn:barelyrobustguarantee}, we have $\Prob_{x\sim D}\insquare{R_1}\geq \beta$ and therefore $\Prob_{x \sim D}\insquare{\Bar{R}_1}\leq 1-\beta$. 

When $t>1$, again by \prettyref{eqn:barelyrobustguarantee}, we have that $\Prob_{x\sim D_t}\insquare{R_t}\geq \beta$ and therefore, by \prettyref{eqn:cond-dist}, $\Prob_{x \sim D}\insquare{\Bar{R}_t | \Bar{R}_{1:t-1}}=\Prob_{x \sim D_t}\insquare{\Bar{R}_t}\leq 1-\beta$. Finally, by the inductive hypothesis and Bayes' rule, we get
\begin{equation}
\label{eqn:boosted-robustness}
    \Prob_{x \sim D}\insquare{\Bar{R}_{1:t}} = \Prob_{x \sim D}\insquare{\Bar{R}_t | \Bar{R}_{1:t-1}}\Prob_{x \sim D}\insquare{\Bar{R}_{1:t-1}}\leq \inparen{1-\beta}\inparen{1-\beta}^{t-1}=\inparen{1-\beta}^t.
\end{equation}

\paragraph{Analysis of Robust Risk.} For each $1\leq t \leq T$, let $$A_t = \SET{x \in \bar{R}_t: \exists z\in \calU(x) \text{ s.t. }G_{h_{t}(z)}\neq \perp \wedge \forall_{t'<t} G_{h_{t'}}(z)=\perp}$$
denote the \emph{non-robust} region of classifier $h_t$ where the \emph{selective} classifier $G_{h_t}$ does not abstain but all selective classifiers $G_{h_1},\dots, G_{h_{t-1}}$ abstain. By the law of total probability, we can analyze the robust risk of the cascade predictor $\CAS(h_1,\dots,h_T)$ by partitioning the space into the \emph{robust} and \emph{non-robust} regions of $h_1,\dots,h_{T}$. Specifically, by the structure of the cascade predictor $\CAS(h_{1:T})$, each  $x\sim D$ such that $\exists z\in\calU(x)$ where $\CAS(h_{1:T})(z) \neq c(x)$ satisfies the following condition:
\begin{equation*}
\begin{split}
\exists_{z\in \calU(x)}\hspace{-0.05cm}:\hspace{-0.05cm} \CAS(h_{1:T})(z) \neq c(x) \Rightarrow \inparen{\exists_{1\leq t\leq T}}\inparen{\exists_{z\in \calU(x)}}\hspace{-0.05cm}:\hspace{-0.05cm} \forall_{t'<t} G_{h_t'}(z) = \perp \wedge G_{h_t}(z)\neq \perp \wedge G_{h_t}(z)\neq c(x).
\end{split}
\end{equation*}
Thus, each  $x\sim D$ such that $\exists z\in\calU(x)$ where $\CAS(h_{1:T})(z) \neq c(x)$ can be mapped to one (or more) of the following regions:
\[\underbrace{R_1 \vee A_1}_{G_{h_1}\text{ does not abstain}} ~~|~~ \underbrace{\bar{R}_1 \wedge \inparen{R_2 \vee A_2}}_{_{G_{h_2}\text{ does not abstain}}} ~~|~~ \bar{R}_{1:2} \wedge \inparen{R_3 \vee A_3} ~~|~~ \dots ~~|~~ \underbrace{\bar{R}_{1:T-1} \wedge \inparen{R_T \vee A_T}}_{G_{h_T}\text{ does not abstain}} ~~|~~ \bar{R}_{1:T}.\]
We will now analyze the robust risk based on the above decomposition:
\begin{equation}
\label{eqn:cas-risk-2}
\begin{split}
    &\Prob_{x\sim D} \insquare{ \exists z\in \calU(x): \CAS(h_{1:T})(z) \neq c(x) }\\
    &\leq \sum_{t=1}^{T+1} \Prob_{x\sim D} \insquare{\bar{R}_{1:t-1}\wedge \inparen{R_t \vee A_t}\wedge \inparen{\exists z\in \calU(x): \CAS(h_{1:T})(z) \neq c(x)}}\\
    &\leq \sum_{t=1}^{T} \Prob_{x\sim D} \insquare{\bar{R}_{1:t-1}\wedge \inparen{R_t \vee A_t}\wedge \inparen{\exists z\in \calU(x): \CAS(h_{1:T})(z) \neq c(x)}} + \Prob_{x\sim D} \insquare{\bar{R}_{1:T}}\\
    &\overset{(i)}{\leq} \sum_{t=1}^{T} \Prob_{x\sim D} \insquare{\bar{R}_{1:t-1}\wedge \inparen{R_t \vee A_t}\wedge \inparen{\exists z\in \calU(x): G_{h_t}(z)\neq \perp \wedge G_{h_t}(z)\neq c(x)}} + \Prob_{x\sim D} \insquare{\bar{R}_{1:T}}\\
    &= \sum_{t=1}^{T} \Prob_{x\sim D} \insquare{\bar{R}_{1:t-1}}\Prob_{x\sim D}\insquare{\inparen{R_t \hspace{-0.05cm}\vee\hspace{-0.05cm} A_t}\wedge \inparen{\exists z\in \calU(x)\hspace{-0.05cm}:\hspace{-0.05cm} G_{h_t}(z)\neq \perp \wedge G_{h_t}(z)\neq c(x)}|\bar{R}_{1:t-1}} \hspace{-0.05cm}+\hspace{-0.05cm} \Prob_{x\sim D} \insquare{\bar{R}_{1:T}}\\
    &\overset{(ii)}{\leq} \sum_{t=1}^{T} \inparen{1-\beta}^{t-1}\epsilon' + (1-\beta)^T=\epsilon'\sum_{t=1}^{T} \inparen{1-\beta}^{t-1}=\epsilon'\frac{1-(1-\beta)^{T}}{1-(1-\beta)} + (1-\beta)^T\leq \frac{\epsilon'}{\beta} + (1-\beta)^T,
\end{split}
\end{equation}
where inequality $(i)$ follows from the definitions of $\bar{R}_{1:t-1}$, $R_t$, and $A_t$, and inequality $(ii)$ follows from \prettyref{eqn:boosted-robustness} and \prettyref{eqn:naterr-total}. It remains to choose $T$ and $\epsilon'$ such that the robust risk is at most $\epsilon$.

\paragraph{Sample and Oracle Complexity.} It suffices to choose $T=\frac{\ln(2/\epsilon)}{\beta}$ and $\epsilon'=\frac{\epsilon\beta}{2}$. We next analyze the sample complexity. Fix an arbitrary round $1\leq t\leq T$. In order to obtain a \emph{good} predictor $h_t$ from learner $\bbA$ satisfying \prettyref{eqn:barelyrobustguarantee}, we need to draw $m_\bbA(\beta,\beta\epsilon/2, \delta/2T)$ samples from $D_t$. We do this by drawing samples from the original distribution $D$ and doing rejection sampling. Specifically, let $m=\max\SET{m_\bbA(\beta, \frac{\epsilon\beta}{2},\frac{\delta}{2T}), 4 \ln\inparen{\frac{2T}{\delta}}}$ (as defined in Step 1). Then, for each $1\leq i \leq m$, let $X_{t,i}$ be a random variable counting the number of samples $(x,y)$ drawn from $D$ until a sample $(x,y)\in \bar{R}_{1:t-1}$ is obtained. Notice that $X_{t,i}$ is a geometric random variable with expectation $1/p_t$ where $p_t=\Prob_{x \sim D}\insquare{\Bar{R}_{1:t}}$. Then, the expected number of samples drawn from $D$ in round $t$ is $\Ex\insquare{\sum_{i=1}^{m}X_{t,i}}=\frac{m}{p_t}$. By applying a standard concentration inequality for the sums of i.i.d.~geometric random variables \citep{brown2011wasted}, we get
\[\Prob\insquare{\sum_{i=1}^{m} X_{t,i} > 2 \frac{m}{p_t}} \leq e^{-\frac{2m(1-1/2)^2}{2}}=e^{-\frac{m}{4}}\leq \frac{\delta}{2T},\]
where the last inequality follows from our choice of $m$. By a standard union bound, we get that with probability at least $1-\frac{\delta}{2}$, the total number of samples 
$\sum_{t=1}^{T}\sum_{i=1}^{m}X_{t,i} \leq \sum_{t=1}^{T}2\frac{m}{p_t}=2m\sum_{t=1}^{T}\frac{1}{p_t}\leq \frac{4mT}{\epsilon}$.
\end{proof}

Before proceeding with the proof of \prettyref{cor:stronger-boosting}, we state the following guarantee that \alphaBoost~provides for boosting weakly robust learners. Its proof is deferred to \prettyref{app:stronger-boosting}.

\begin{lem}[\cite{pmlr-v99-montasser19a}]
\label{lem:weak-robust-learner}
For any perturbation set $\calU$, \alphaBoost~$(\epsilon,\delta)$-robustly-learns any target concept $c$ w.r.t.~$\calU$ using $T$ black-box oracle calls to any $(\frac{1}{3},\frac{1}{3})$-robust-learner $\bbB$ for $c$ w.r.t.~$\calU$, with total sample complexity  
\[m(\epsilon,\delta) = O\inparen{\frac{m_{\bbB}\inparen{1/3,1/3}}{\epsilon}\log^2\inparen{\frac{m_{\bbB}\inparen{1/3,1/3}}{\epsilon}}+\frac{\log(1/\delta)}{\epsilon}},\]
and oracle calls $T = O(\log(m)\inparen{\log(1/\delta)+\log\log m})$.
\end{lem}

We are now ready to proceed with the proof of \prettyref{cor:stronger-boosting}.

\begin{proof}[of \prettyref{cor:stronger-boosting}.]
The main idea is to perform two layers of boosting. In the first layer, we use \RoBoost\ to get a \emph{weak} robust learner for $c$ w.r.t.~$\calU$ from a \emph{barely} robust learner $\bbA$ for $c$ w.r.t.~$\calU^{-1}(\calU)$. Then, in the second layer, we use \alphaBoost~to boost \RoBoost\ from a \emph{weak} robust learner to a \emph{strong} robust learner for $c$ w.r.t.~$\calU$.

Let $\bbA$ be a $(\beta,\frac{\beta}{6},\frac{\beta}{6\ln(6)})$-barely-robust-learner $\bbA$ for $c$ w.r.t.~$\calU^{-1}(\calU)$. Let $(\epsilon_0, \delta_0, T_0)=(\frac{1}{3},\frac{1}{3}, \frac{\ln(6)}{\beta})$, and observe that $\bbA$ is a $(\beta,\frac{\beta\epsilon_0}{2},\frac{\delta_0}{2T_0})$-barely-robust-learner for $c$ w.r.t.~$\calU^{-1}(\calU)$. By \prettyref{thm:boost-realizable}, \RoBoost\ $(\epsilon_0,\delta_0)$-robustly-learns $c$ w.r.t.~$\calU$ using $T_0$ black-box oracle calls to $\bbA$, with sample complexity $m_0 = O\inparen{\frac{\max\SET{m_{\bbA},4\ln\inparen{\frac{6\ln(6)}{\beta}}}}{\beta}}$.
Finally, by \prettyref{lem:weak-robust-learner}, \alphaBoost~$(\epsilon,\delta)$-robustly-learns $c$ w.r.t.~$\calU$ using $m(\epsilon,\delta)=O\inparen{\frac{m_0}{\epsilon}\log\inparen{\frac{m_0}{\epsilon}}+\frac{\log(1/\delta)}{\epsilon}}$ samples, and $O(\log m)$ black-box oracle calls to \RoBoost. 
\end{proof}

\subsection{Boosting Robustness with Unlabeled Data}

Prior work has shown that unlabeled data can improve adversarially robust generalization in practice \citep{DBLP:conf/nips/AlayracUHFSK19,DBLP:conf/nips/CarmonRSDL19}, and there is also theoretical work quantifying the benefit of unlabeled data for robust generalization \citep{DBLP:conf/icml/AshtianiPU20}. In this section, we highlight yet another benefit of unlabeled data for adversarially robust learning. Specifically, we show that it is \emph{possible} to boost robustness by relying only on \emph{unlabeled} data. 

We will start with some intuition first. For an unknown distribution $D_c$, imagine having access to a \emph{non-robust} classifier $h$ that makes no mistakes on natural examples, i.e., $\Prob_{x\sim D}\insquare{h(x)\neq c(x)}=0$ but $\Prob_{x\sim D}\insquare{\exists z\in \calU(x):h(z)\neq h(x)}=1$. Now, in order to learn a robust classifier, we can use \RoBoost\ where in each round of boosting we sample unlabeled data from $D$ (label it with $h$) and call a barely robust learner $\bbA$ on this \emph{pseudo-labeled} dataset. 

This highlights that perhaps robustness can be boosted using only unlabeled data if we have access to a good \emph{pseudo-labeler} $h$ that makes few mistakes on natural examples from $D$. But in case that $\Prob_{x\sim D}\insquare{h(x)\neq c(x)}=\epsilon$ for some small $\epsilon > 0$, it no longer suffices to use a barely robust learner $\bbA$ for $c$, but rather we need a more powerful learner that is tolerant to the noise in the labels introduced by $h$. Formally, we introduce the following noise-tolerant barely robust learner:

\begin{defn} [Noise-Tolerant Barely Robust Learner]
\label{def:noise-barelyrobust}
Learner $\bbA$ $(\eta, \beta, \epsilon, \delta)$-barely-robustly-learns a concept $c:\calX\to\calY$ w.r.t. $\calU^{-1}(\calU)$ if there exists $m(\eta, \beta,\epsilon,\delta)\in \bbN$ such that for any distribution $D$ over $\calX$ satisfying $\Prob_{x\sim D}\insquare{\exists z\in \calU(x): c(z)\neq c(x)}=0$ and any $h:\calX\to\calY$ where $\Prob_{x\sim D}\insquare{h(x)\neq c(x)}\leq \eta$, w.p.~at least $1-\delta$ over $S\sim D^m_h$, $\bbA$ outputs a predictor $\hat{h}=\bbA(S)$ satisfying:
\[\Prob_{x\sim D}\insquare{x\in \Rob_{\calU^{-1}(\calU)}(\hat{h})}\geq \beta \text{ and }\Prob_{x\sim D}\insquare{\hat{h}(x)\neq c(x)}\leq \Prob_{x\sim D}\insquare{h(x)\neq c(x)} + \epsilon \leq \eta + \epsilon.\]
\end{defn}

In \prettyref{thm:boost-unlabeled}, we show that given a \emph{noise-tolerant} barely robust learner $\bbA$ for some unknown target concept $c$ (according to \prettyref{def:noise-barelyrobust}), it is \emph{possible} to strongly robustly learn $c$ with \URoBoost\ by making black-box oracle calls to $\bbA$. 

\begin{algorithm2e}[H]
\caption{\textup{$\beta$-URoBoost} --- Boosting Robustness with Unlabeled Data}
\label{alg:unlabeled}
\SetKwInput{KwInput}{Input}                % Set the Input
\SetKwInput{KwOutput}{Output}              % set the Output
\SetKwFunction{pred}{CAS}
\SetKwFunction{rej}{RejectionSampling}
\DontPrintSemicolon
  \KwInput{Sampling oracle for distribution $D_c$, black-box noise-tolerant barely-robust learner $\bbA$.}
Draw $m=m_\bbA(\eta,\beta, \frac{\beta\epsilon}{4},\frac\delta2)$ labeled samples $S=\SET{(x_1,y_1),\dots, (x_m,y_m)}\sim D_c$.\;
Call learner $\bbA$ on $S$ and let predictor $\hat{h}=\bbA(S)$ be its output.\;
Call \RoBoost\ with access to labeled samples from $D_{\hat{h}}$ (i.e., $(x,\hat{h}(x))\sim D_{\hat{h}}$), and black-box $(\eta,\beta,\frac{\epsilon\beta}{4},\frac{\delta}{2T})$-noise-tolerant-barely-robust-learner $\bbA$.\;
\KwOutput{The cascade predictor $\CAS(h_1,\dots,h_T)$.}
\end{algorithm2e}

\begin{thm}
\label{thm:boost-unlabeled}
For any perturbation set $\calU$, \URoBoost\ $(\epsilon,\delta)$-robustly-learns any target concept $c$ w.r.t.~$\calU$ using $T+1\leq \frac{\ln(2/\epsilon)}{\beta}+1$ black-box oracle calls to any $(\eta,\beta,\frac{\beta\epsilon}{4},\frac{\delta}{2T})$-barely-robust learner $\bbA$ for $c$ w.r.t.~$\calU^{-1}(\calU)$, with \emph{labeled} sample complexity of $m_\bbA(\eta,\beta, \frac{\beta\epsilon}{4},\frac\delta2)$ and \emph{unlabeled} sample complexity of at most
\[\frac{4T\max\SET{m_\bbA(\eta,\beta, \frac{\beta\epsilon}{4},\frac{\delta}{4T}), 4 \ln\inparen{\frac{4T}{\delta}}}}{\epsilon}.\]
\end{thm}

\begin{proof}
Let $\calU$ be an arbitrary perturbation set, and $\bbA$ a $(\eta,\beta, \epsilon,\delta)$-barely-robust learner for some \emph{unknown} target concept $c:\calX\to\calY$ with respect to $\calU^{-1}(\calU)$. We will show that \URoBoost\  $(\epsilon,\delta)$-robustly-learns $c$ with respect to $\calU$. Let $D$ be some unknown distribution over $\calX$ such that $\Prob_{x\sim D}\insquare{\exists z\in \calU(x): c(z)\neq c(x)}=0$.

By Step 1 and Step 2 and the guarantee of learner $\bbA$ (see \prettyref{def:noise-barelyrobust}), with probability at least $1-\frac\delta2$ over $S^m\sim D_c$, it holds that
\[\Prob_{x\sim D}\insquare{\hat{h}(x)\neq c(x)}\leq \frac{\beta\epsilon}{4}.\]

That is, with high probability, $\hat{h}$ is a predictor with low error on \emph{natural} examples.

\paragraph{Pseudo labeling.} In Step 3, \URoBoost\ essentially runs \RoBoost\  using unlabeled samples from $D$ that are labeled with the predictor $\hat{h}$. Thus, we can view this as robustly learning the concept $\hat{h}$ which is only an approximation of the true concept $c$ that we care about. Since the noise tolerance $\eta\geq \frac{\beta\epsilon}{4}$, it follows by the guarantees of learner $\bbA$ (see \prettyref{def:noise-barelyrobust}) and \prettyref{eqn:cas-risk-2}, that
\[\Prob_{x\sim D} \insquare{ \exists z\in \calU(x): \CAS(h_{1:T})(z) \neq c(x) } \leq \frac{\Prob_{x\sim D}\insquare{\hat{h}(x)\neq c(x)}}{\beta}+ \frac{\beta\epsilon}{4\beta} + (1-\beta)^T\leq \frac{\epsilon}{4} + \frac{\epsilon}{4}+ \frac{\epsilon}{2} \leq \epsilon.\]
\end{proof}

\section{The Necessity of Barely Robust Learning}
\label{sec:backward-direction}

We have established in \prettyref{sec:forward-direction} (\prettyref{thm:boost-realizable}) that our proposed notion of barely robust learning in \prettyref{def:barelyrobust} \emph{suffices} for strongly robust learning. But is our notion actually \emph{necessary} for strongly robust learning? In particular, notice that in our proposed notion of barely robust learning in \prettyref{def:barelyrobust}, we require $\beta$-robustness with respect to a ``larger'' perturbation set $\calU^{-1}(\calU)$, instead of the actual perturbation set $\calU$ that we care about. Is this necessary? or can we perhaps boost robustness even with the weaker guarantee of $\beta$-robustness with respect to $\calU$?

In this section, we answer this question in the negative. First, we provably show in \prettyref{thm:stronglytobarely} that strongly robust learning \wrt $\calU$ implies barely robust learning \wrt $\calU^{-1}(\calU)$. This indicates that our proposed notion of barely robust learning \wrt $\calU^{-1}(\calU)$ (\prettyref{def:barelyrobust}) is \emph{necessary} for strongly robust learning \wrt $\calU$. Second, we provably show in \prettyref{thm:Uisnotgood} that barely robust learning \wrt $\calU$ does \emph{not} imply strongly robust learning \wrt $\calU$ when the robustness parameter $\beta\leq \frac{1}{2}$ which is the main regime of interest that we study in this work. 

\begin{thm}
\label{thm:stronglytobarely}
For any $\calU$, learner $\bbB$, and $\epsilon\in(0,1/4)$, if $\bbB$ $(\epsilon,\delta)$-robustly learns some unknown target concept $c$ w.r.t.~$\calU$, then there is a learner $\Tilde{\bbB}$ that $(\frac{1-\epsilon}{2}, 2\epsilon, 2\delta)$-barely-robustly-learns $c$ w.r.t. $\calU^{-1}(\calU)$.%, where $\beta=(1-\epsilon)/2$. 
\end{thm}

We briefly describe the high-level strategy here. The main idea is to convert a strongly robust learner $\bbB$ \wrt $\calU$ to a barely robust learner $\Tilde{\bbB}$ \wrt $\calU^{-1}(\calU)$. We do this with a simple \emph{expansion} trick that modifies a predictor $h$ robust \wrt $\calU$ to a predictor $g$ robust \wrt $\calU^{-1}(\calU)$. For each label $y\in\calY$, we do this expansion conditional on the label to get a predictor $g_y$ that is robust w.r.t. $\calU^{-1}(\calU)$ but only in the region of $\calX$ where $h$ predicts the label $y$ robustly w.r.t. $\calU$. This is described in the following key Lemma. We then use fresh samples to select predictor a $g_y$ whose label $y$ occurs more often. 

\begin{lem}
\label{lem:helper}
For any distribution $D$ over $\calX$ and any concept $c:\calX\to \SET{\pm1}$, given a predictor $\hat{h}:\calX\to \SET{\pm 1}$ such that $\Risk_\calU(\hat{h};D_c)\leq \epsilon$ for some $\epsilon \in (0,1/4)$, then for each $y\in \SET{\pm 1}$, the predictor $g_y$ defined for each $x\in \calX$ as
$$g_y(x) \triangleq y \text{ \normalfont iff } x\in \underset{\Tilde{x}\in \Rob_{\calU}(\hat{h})\wedge \hat{h}(\Tilde{x})=y}{\bigcup~\calU^{-1}(\calU(\Tilde{x}))}\text{ satisfies}$$
\[
    \Prob_{x\sim D}\insquare{g_y(x)\neq c(x)} \leq 2\epsilon\text{ and } \Prob_{x\sim D}\insquare{x\in \Rob_{\calU^{-1}(\calU)}(g_y)}\geq (1-\epsilon)\Prob_{x\sim D}\insquare{\hat{h}(x)=y\big| x\in \Rob_{\calU}(\hat{h})}.
\]
\end{lem}

\begin{proof}
Without loss of generality, let $y=+1$. Let $x\in {\rm supp}(D)$ such that $x\in \Rob_\calU(\hat{h})$. In case $\hat{h}(x)=+1$, then by definition of $g_{+}$, since $x\in\calU^{-1}(\calU(x))$, it holds that $g_{+}(x)=+1$. In case $\hat{h}(x)=-1$, then $\lnot\exists \Tilde{x}\in \Rob_{\calU}(\hat{h})$ such that $\hat{h}(\Tilde{x})=+1$ and $\calU(\Tilde{x})\cap \calU(x)\neq \emptyset$, which implies that $x\notin \underset{\Tilde{x}\in \Rob_{\calU}(\hat{h})\wedge \hat{h}(\Tilde{x})=+1}{\bigcup~\calU^{-1}(\calU(\Tilde{x}))}$, and therefore, $g_+(x)=-1$. This establishes that in the robust region of $\hat{h}$, $\Rob_\calU(\hat{h})$, the predictions of $g_+$ on \emph{natural} examples $x\sim D$ are equal to the predictions of $\hat{h}$. We will use this observation, in addition to the fact that the robust risk of $\hat{h}$ is small ($\Prob_{x\sim D} \insquare{\exists z\in \calU(x): \hat{h}(z)\neq c(x)}\leq \epsilon$) to show that the error of $g_+$ on natural examples is small. Specifically, by law of total probability,
\begin{align*}
    \Prob_{x\sim D}\insquare{ g_+(x)\neq c(x) } &\hspace{-0.05cm}=\hspace{-0.05cm} \Prob_{x\sim D}\insquare{g_+(x)\neq c(x) \wedge x\in \Rob_{\calU}(\hat{h})} \hspace{-0.05cm}+\hspace{-0.05cm} \Prob_{x\sim D}\insquare{g_+(x)\neq c(x) \wedge x\notin \Rob_{\calU}(\hat{h})}\\
    &= \Prob_{x\sim D}\insquare{\hat{h}(x)\neq c(x) \wedge x\in \Rob_{\calU}(\hat{h})} + \Prob_{x\sim D}\insquare{g_+(x)\neq c(x) \wedge x\notin \Rob_{\calU}(\hat{h})}\\ 
    &\leq \Prob_{x\sim D}\insquare{\hat{h}(x)\neq c(x) \wedge x\in \Rob_{\calU}(\hat{h})} + \Prob_{x\sim D}\insquare{x\notin \Rob_{\calU}(\hat{h})} \leq \epsilon + \epsilon = 2\epsilon.
\end{align*}
Finally, observe that for any $x\in \Rob_\calU(\hat{h})$ such that $\hat{h}(x)=+1$, by definition of $g_+$, it holds that $x\in \Rob_{\calU^{-1}(\calU)}(g_+)$, thus
\begin{align*}
    \Prob_{x\sim D}\insquare{ x\in \Rob_{\calU^{-1}(\calU)}(g_+) } &\geq \Prob_{x\sim D}\insquare{x\in \Rob_\calU(\hat{h}) \wedge \hat{h}(x)=+1}\\
    &= \Prob_{x\sim D}\insquare{x\in \Rob_\calU(\hat{h})}\Prob_{x\sim D}\insquare{ \hat{h}(x)=+1\big| x\in \Rob_\calU(\hat{h})}\\
    &\geq (1-\epsilon) \Prob_{x\sim D}\insquare{ \hat{h}(x)=+1\big| x\in \Rob_\calU(\hat{h})}.
\end{align*}
\end{proof}

We are now ready to proceed with the proof of \prettyref{thm:stronglytobarely}.
\begin{proof}[of \prettyref{thm:stronglytobarely}.]
Let $\calU$ be an arbitrary perturbation set, and $\bbB$ an $(\epsilon,\delta)$-robust learner for some \emph{unknown} target concept $c:\calX\to \calY$ with respect to $\calU$. We will construct another learner $\Tilde{\bbB}$ that $(\beta, 2\epsilon, 2\delta)$-barely-robustly-learns $c$ \wrt $\calU^{-1}(\calU)$, with $\beta=(1-\epsilon)/2$. Let $D$ be some unknown distribution over $\calX$ that is robustly realizable: $\Prob_{x\sim D}\insquare{\exists z\in \calU(x): c(z)\neq c(x)}=0$.

\paragraph{Description of $\Tilde{\bbB}$.}Sample $S\sim D_c^{m_{\bbB}(\epsilon,\delta)}$, and run learner $\bbB$ on $S$. Let $\hat{h}=\bbB(S)$ be the predictor returned by $\bbB$. Let $\Tilde{m}\geq \frac{64}{9}\ln(1/\delta)$. For each $1\leq i\leq \Tilde{m}$, consider the following process: draw an example $(x,y)\sim D_c$. If $x\in \Rob_\calU(\hat{h})$ terminate, otherwise repeat the process again. Let $\Tilde{S}=\SET{(x_1,y_1),\dots,(x_{\Tilde{m}},y_{\Tilde{m}})}$ be the sample resulting from this process. Calculate $M_+ = \frac{1}{\abs{\Tilde{S}}}\sum_{x\in \Tilde{S}} \ind[\hat{h}(x) = +1]$. If $M_+ \geq 1/2$, output $g_+$, otherwise, output $g_-$ (as defined in \prettyref{lem:helper}).
\paragraph{Analysis.} With probability at least $1-\delta$ over $S\sim D_c^m$, $\hat{h}$ has small robust risk: $\Risk_\calU(\hat{h};D_c)\leq \epsilon$.
\prettyref{lem:helper} implies then that for each $y\in\SET{\pm1}$, $g_y$ satisfies:
\begin{equation*}
    \Prob_{x\sim D}\insquare{g_y(x)\neq c(x)} \leq 2\epsilon\text{ and } \Prob_{x\sim D}\insquare{x\in \Rob_{\calU^{-1}(\calU)}(g_y)}\geq (1-\epsilon)\Prob_{x\sim D}\insquare{\hat{h}(x)=y\big| x\in \Rob_{\calU}(\hat{h})}.
\end{equation*}

It remains to show that with probability at least $1-\delta$ over $\Tilde{S}\sim D^{\Tilde{m}}$, for $g_{\hat{y}}$ returned by $\Tilde{\bbB}$:
\[\Prob_{x\sim D}\insquare{\hat{h}(x)=\hat{y}\big| x\in \Rob_{\calU}(\hat{h})}\geq \frac{1}{2}.\]
Observe that by the rejection sampling mechanism of $\Tilde{\bbB}$, $\Tilde{S}$ is a sample from the region of distribution $D$ where $\hat{h}$ is robust. Furthermore, we know that
\[\max\SET{\Prob_{x\sim D}\insquare{\hat{h}(x)=+1\big| x\in \Rob_{\calU}(\hat{h})}, \Prob_{x\sim D}\insquare{\hat{h}(x)=-1\big| x\in \Rob_{\calU}(\hat{h})}} \geq \frac{1}{2}.\]
Without loss of generality, suppose that $p=\Prob_{x\sim D}\insquare{\hat{h}(x)=+1\big| x\in \Rob_{\calU}(\hat{h})}\geq 1/2$. Then, the failure event is that $\Tilde{\bbB}$ outputs $g_-$, i.e. the event that $M_+<\frac{1}{2}$. By a standard application of the Chernoff bound, we get that
\[\Prob_{\Tilde{S}}\insquare{M_+<\frac{1}{2}}\leq e^{\frac{-\Tilde{m}p\frac{1}{4}\inparen{\frac{1}{2}-\frac{1}{p}}^2}{2}}\leq e^{-\frac{9\Tilde{m}}{64}}\leq \delta,\]
where the last inequality follows from the choice of $\Tilde{m}$ in the description of $\Tilde{\bbB}$. 

Finally, to conclude, observe that the sample complexity of learner $\Tilde{\bbB}$ is equal to $m_\bbB(\epsilon,\delta)$ plus the number of samples drawn from $D$ to construct $\Tilde{S}$. For each $1\leq i \leq \Tilde{m}$, let $X_i$ be the number of samples drawn from $D$ until a sample from the robust region $\Rob_\calU(\hat{h})$ was observed. Note that $X_i$ is a geometric random variable with mean at most $1/(1-\epsilon)$. By a standard concentration inequality for the sums of i.i.d. geometric random variables \citep{brown2011wasted},
\[\Prob\insquare{\sum_{i=1}^{\Tilde{m}} X_i > 2\frac{\Tilde{m}}{1-\epsilon}} \leq e^{-\frac{\Tilde{m}}{4}}\leq \delta,\]
where the last inequality follows from the choice of $\Tilde{m}$ in the description of $\Tilde{\bbB}$. Thus, with probability at least $1-\delta$, the total sample complexity is $m_\bbB(\epsilon,\delta)+\frac{2\Tilde{m}}{1-\epsilon}$. This concludes that learner $\Tilde{\bbB}$ $(\beta, 2\epsilon, 2\delta)$-barely-robustly-learns $c$ w.r.t. $\calU^{-1}(\calU)$, where $\beta=(1-\epsilon)/2$.
\end{proof}

As mentioned earlier, \prettyref{thm:stronglytobarely} still leaves open the question of whether the \emph{weaker} requirement of barely robust learning \wrt $\calU$ suffices for strongly robust learning \wrt $\calU$. We show next that this weaker requirement is \emph{not} sufficient. 

\begin{thm}
\label{thm:Uisnotgood}
There is an instance space $\calX$, a perturbation set $\calU$, and a class $\calC$ such that $\calC$ is $(\beta=\frac{1}{2},\epsilon=0,\delta)$-barely-robustly-learnable \wrt $\calU$, but $\calC$ is {\em not} $(\epsilon,\delta)$-robustly-learnable \wrt $\calU$ for any $\epsilon<1/2$.
\end{thm}
Before proceeding with the proof of \prettyref{thm:Uisnotgood}, we briefly sketch the high-level argument. In order to show this impossibility result, we construct a collection of distributions and show that this collection is \emph{barely} robustly learnable \wrt $\calU$ with robustness parameter $\beta=\frac{1}{2}$ and natural error $\epsilon=0$ using a randomized predictor. We also show that it is not possible to robustly learn this collection with robust risk strictly smaller than $1/2$. The second part is shown by relying on a necessary condition for strongly robust learning proposed by \citet{pmlr-v99-montasser19a} which is the finiteness of the robust shattering dimension:

\begin{defn}[Robust Shattering Dimension]
\label{def:robustshatter-dim-orig}
A sequence $z_1,\ldots,z_k \in \calX$ is said to be \emph{$\calU$-robustly shattered} by $\calC$ if $\exists x^+_1,x^-_1,\dots,x^{+}_k,x^{-}_k\in\calX$ such that $\forall i\in[k], z_i\in\calU(x^+_i)\cap\calU(x^-_i)$ and $\forall y_1,\dots,y_k \in \SET{\pm 1}:\exists h\in \calC$ such that $h(z')=y_i \forall z'\in \calU(x^{y_i}_i), \forall 1\leq i\leq k$. The \emph{$\calU$-robust shattering dimension} $\dim_{\calU}(\calC)$ is defined as the largest $k$ for which there exist $k$ points $\calU$-robustly shattered by $\calC$.
\end{defn}

The following lemma due to \citet{pmlr-v99-montasser19a} states that finite robust shattering dimension $\dim_\calU(\calC)$ is \emph{necessary} for strongly robustly learning $\calC$ \wrt $\calU$. 

\begin{lem}[\citet{pmlr-v99-montasser19a}]
\label{lem:lowerbound}
For any class $\calC$ and any perturbation set $\calU$, 
$\calC$ is $(\epsilon,\delta)$-robustly-learnable \wrt $\calU$ only if $\dim_\calU(\calC)$ is finite. 
\end{lem}

We are now ready to proceed with the proof of \prettyref{thm:Uisnotgood}.

\begin{proof}[of \prettyref{thm:Uisnotgood}.]
Pick three infinite unique sequences $(x^{+}_n)_{n\in \bbN}$, $(x^{-}_n)_{n\in \bbN}$, and $(z_n)_{n\in \bbN}$ from $\bbR^2$ such that for each $n \in \bbN: x^+_n=(n,1), x^-_n=(n,-1), z_n=(n,0)$, and let $\calX= \cup_{n\in \bbN} \SET{x^+_n, x^-_n, z_n}$. We now describe the construction of the perturbation set $\calU$. For each $n\in \bbN$, let $\calU(x^+_n)=\SET{x^+_n,z_n}, \calU(x^-_n)=\SET{x^-_n,z_n},\text{ and } \calU(z_{n})=\SET{z_{n},x^+_n,x^-_n}$.

We now describe the construction of the concept class $\calC$. For each $\by \in \SET{\pm 1}^{\bbN}$ define $h_\by: \calX \to \calY$ to be:
\begin{equation}
\label{eqn:class-construction}
    \forall n\in \bbN : h_\by(z_n)=y_n \wedge h_\by(x_n^+)=+1 \wedge h_\by(x_n^-)=-1.
\end{equation}
Let $\calC=\SET{h_\by: \by \in \SET{\pm 1}^{\bbN}}$. Observe that by construction of $\calU$ and $\calC$, $\calC$ robustly shatters the sequence $(z_n)_{n\in \bbN}$ with respect to $\calU$ (see \prettyref{def:robustshatter-dim-orig}), and therefore, the robust shattering dimension of $\calC$ with respect to $\calU$, $\dim_\calU(\calC)$, is infinite. Thus, \prettyref{lem:lowerbound} implies that $\calC$ is not $(\epsilon,\delta)$-\emph{strongly}-robustly-learnable \wrt $\calU$.

We will now show that there is a simple learner $\bbA$ that $(\beta,\epsilon,\delta)$-\emph{barely}-robustly-learns $\calC$ \wrt $\calU$, with robustness parameter $\beta=\frac{1}{2}$ and natural error $\epsilon=0$. Specifically, $\bbA$ samples a bitstring $\Tilde{\by} \in \SET{\pm 1}^{\bbN}$ uniformly at random, and outputs the classifier $h_{\Tilde{\by}}$. Learner $\bbA$ will not require any data as input.

We now proceed with analyzing the performance of learner $\bbA$. Let $h_{\by}\in \calC$ be some unknown target concept and $D$ be some unknown distribution over $\calX$ that is robustly realizable: $\Prob_{x\sim D}\insquare{\exists{z\in\calU(x)}: h_\by(z)\neq y}=0$. Since $D$ is robustly realizable, by construction of $\calU$ and $\calC$, this implies that 
\begin{equation}
    \label{eqn:mass}
    \forall n \in \bbN: D(z_n)=0 ~~\text{and}~~D(x^{-y_n}_n)=0.
\end{equation}
This is because $\calU(z_n)=\SET{z_n,x^+_n,x^-_n}$ and \prettyref{eqn:class-construction} implies that $h_\by$ is not robust on $z_n$ since $h_\by(x^+_n)\neq h_{\by}(x^-_n)$, also $\calU(x^+_n)\cap \calU(x^-_n)=\SET{z_n}$ and since $h_\by(z_n)=y_n$ this implies that $h_\by$ is not robust on $x^{-y_n}_n$. 
\prettyref{eqn:mass} and \prettyref{eqn:class-construction} together imply that the random classifier $h_{\tby}\in \calC$ selected by learner $\bbA$ has \emph{zero} error on \emph{natural} examples: with probability 1 over $\tby$, $\Prob_{x\sim D}\insquare{h_{\tby}(x)\neq h_{\by}(x)}=0$. 

We now turn to analyzing the \emph{robust} risk of learner $\bbA$,
\begin{align*}
    \Ex_{\tby}\insquare{\Ex_{x\sim D}\insquare{\ind\SET{\exists z\in \calU(x):h_{\tby}(z)\neq h_\by(x)}}}&=\Ex_{x\sim D}\insquare{\Ex_{\tby}\insquare{\ind\SET{\exists z\in \calU(x):h_{\tby}(z)\neq h_\by(x)}}}\\
    &= \sum_{n\in \bbN} D(x^{y_n}_n) \Ex_{\tby}\insquare{\ind\SET{\exists z\in \calU(x^{y_n}_n):h_{\tby}(z)\neq h_\by(x^{y_n}_n)}}\\
    &= \sum_{n\in \bbN} D(x^{y_n}_n) \Ex_{\tby}\insquare{\ind\SET{h_{\tby}(z_n)\neq h_\by(x^{y_n}_n)}}\\
    &= \sum_{n\in \bbN} D(x^{y_n}_n) \Ex_{\tby}\insquare{\ind\SET{\Tilde{y}_n\neq y_n}} = \sum_{n\in \bbN} D(x^{y_n}_n) \frac{1}{2} = \frac{1}{2}.
\end{align*}
This implies that in expectation over randomness of learner $\bbA$, it will be robust on half the mass of distribution $D$: $\Ex_{\tby} \Ex_{x\sim D} \ind[ x\in \Rob_{\calU}(h_{\tby}) ]=\frac{1}{2}$.
\end{proof}
\section{Discussion}
\label{sec:dis}

In this paper, we put forward a theory for \emph{boosting} adversarial robustness. We discuss below practical implications and outstanding directions that remain to be addressed.

\paragraph{Practical implications.} Our algorithm \RoBoost~is generic and can be used with any black-box barely robust learner $\bbA$. In the context of deep learning and $\ell_p$ robustness, our results suggest the following: for targeted robustness of radius $\gamma$, use an adversarial learning method \citep[e.g., ][]{DBLP:conf/iclr/MadryMSTV18,DBLP:conf/icml/ZhangYJXGJ19, DBLP:conf/icml/CohenRK19} to learn a neural net $h^1_{\rm NN}$ predictor robust with radius $2\gamma$, then filter the training examples to include {\em only} the ones on which $h^1_{\rm NN}$ is {\em not} robust with radius $2\gamma$, and repeat this process on the filtered examples to learn a second neural net, and so on. Finally, use the cascade of neural nets $\CAS(h^1_{\rm NN},h^2_{\rm NN}, \dots)$ to predict. 

It would be interesting to empirically explore whether adversarial learning methods \citep[e.g., ][]{DBLP:conf/iclr/MadryMSTV18,DBLP:conf/icml/ZhangYJXGJ19, DBLP:conf/icml/CohenRK19} satisfy the \emph{barely} robust learning condition: on each round of boosting, the learning algorithm can shrink the fraction of the training examples on which the predictor from the previous round is not robust on. This is crucial for progress.

\paragraph{Multiclass.} We would like to emphasize that our theory for boosting \emph{barely robust} learners extends seamlessly to multiclass learning problems. In particular, when the label space $\abs{\calY}>2$, we obtain the same guarantees in \prettyref{thm:boost-realizable} using the same algorithm \RoBoost. The other direction of converting an $(\epsilon,\delta)$-robust-learner \wrt $\calU$ to an $(\beta,2\epsilon,2\delta)$-barely-robust-learner \wrt $\calU^{-1}(\calU)$ also holds, using the same technique in \prettyref{thm:stronglytobarely}, but now we get $\beta=\frac{1-\epsilon}{\abs{\calY}}$.

\paragraph{Boosting robustness independent of the error rate.} To achieve robust risk at most $\epsilon$ using a barely robust learner $\bbA$ with robustness parameter $\beta$, our algorithm \RoBoost~requires $\bbA$ to achieve a natural error of $\Tilde{\epsilon}=\frac{\beta\epsilon_0}{2}$ for any constant $\epsilon_0<\frac{1}{2}$ (say $\Tilde{\epsilon}=\frac{\beta}{6}$) (see \prettyref{cor:stronger-boosting}). It would be interesting to resolve whether requiring natural risk $\Tilde{\epsilon}$ that depends on $\beta$ is necessary, or whether it is possible to avoid dependence on $\beta$. Concretely, an open question here is: can we achieve robust risk at most $\epsilon$ using a $(\beta,O(\epsilon),\delta)$-barely-robust-learner instead of requiring $(\beta,O(\beta),\delta)$-barely-robust-learner? It actually suffices to answer the following: given a $(\beta,\frac{1}{3},\delta)$-barely-robust-learner, is it possible to achieve {\em robust} risk at most $\frac{1}{3}$?

\paragraph{Boosting error rate independent of robustness.} A related question is whether it is possible to boost the error while maintaining robustness fixed at some level. For example, given a $(\beta, \frac{1}{3},\delta)$-barely-robust-learner $\bbA$, is it possible to boost this to a $(\beta,\epsilon, \delta)$-barely-robust-learner $\bbB$?
\paragraph{Agnostic setting.} We focused only on boosting robustness in the \emph{realizable} setting, where we assume that the target unknown concept $c$ and the unknown distribution $D$ satisfy $\Risk_\calU(c;D_c)=0$. It would be interesting to explore meaningful formulations of the problem of boosting robustness beyond the realizable setting.

\acks{OM would like to thank Steve Hanneke and Nathan Srebro for insightful discussions. This work was supported in part by DARPA under cooperative agreement HR00112020003.\footnote{The views expressed in this work do not necessarily reflect the position or the policy of the Government and no official endorsement should be inferred. Approved for public release; distribution is unlimited.}
This work was supported in part by the National Science Foundation under grant CCF-1815011.}
\bibliography{learning}

\newpage
\appendix

\section{Auxiliary Lemmas and Proofs for \prettyref{cor:stronger-boosting}}
\label{app:stronger-boosting}

\begin{algorithm2e}
\caption{\alphaBoost~--- Boosting \emph{weakly} robust learners}
\label{alg:alphaboost}
\SetKwInput{KwInput}{Input}                % Set the Input
\SetKwInput{KwOutput}{Output}              % set the Output
\SetKwFunction{FMain}{CycleRobust}
\SetKwFunction{FDisc}{Discretizer}
\DontPrintSemicolon
    \KwInput{Training dataset $S=\SET{(x_1,y_1),\dots, (x_m,y_m)}$, black-box {\em weak} robust learner $\bbB$.}
  \BlankLine
  Set $m_0 =m_{\bbB}(1/3,1/3)$.\;
  Initialize $D_1$ to be uniform over $S$, and set $T=O(\log m)$.\;
  \For{$1 \leq t\leq T$}{
    Sample $S_t\sim D_t^{m_0}$, call learner $\bbB$ on $S_t$, and denote by $h_t$ its output predictor. Repeat this step until $\Risk_\calU(h_t;D_t)\leq 1/3$.\;
    Compute a new distribution $D_{t+1}$ by applying the following update for each $(x,y)\in S$:
            \[ 
                D_{t+1}(\SET{(x,y)}) = \frac{D_t(\SET{(x,y)})}{Z_t} \times \begin{cases} 
                                                              e^{-2\alpha},& \text{if }\ind[\forall z\in \calU(x): h_{t-1}(z) = y]=1;\\
                                                              1, &\text{otherwise,}
                                                           \end{cases}
            \]
            where $Z_t$ is a normalization factor and $\alpha=1/8$.\;
  }
\KwOutput{A majority-vote classifier $\MAJ(h_{1},\dots, h_{T})$.}
\end{algorithm2e}

\begin{lem} [Sample Compression Robust Generalization Guarantee -- \cite{pmlr-v99-montasser19a}]
\label{lem:robust-compression}
For any $k \in \bbN$ and fixed function $\phi : (\calX \times \calY)^{k} \to \calY^{\calX}$, for any distribution $P$ over $\calX \times \calY$ and any $m \in \bbN$, 
for $S = \{(x_{1},y_{1}),\ldots,(x_{m},y_{m})\}$ i.i.d. $P$-distributed random variables,
with probability at least $1-\delta$, 
if $\exists i_{1},\ldots,i_{k} \in \{1,\ldots,m\}$ 
s.t.\ $\hat{R}_{\calU}(\phi((x_{i_{1}},y_{i_{1}}),\ldots,(x_{i_{k}},y_{i_{k}}));S) = 0$, 
then 
\begin{equation*}
\Risk_{\calU}(\phi((x_{i_{1}},y_{i_{1}}),\ldots,(x_{i_{k}},y_{i_{k}}));P) \leq \frac{1}{m-k} (k\ln(m) + \ln(1/\delta)).
\end{equation*}
\end{lem}

\begin{proof}[of \prettyref{lem:weak-robust-learner}]
Let $\bbB$ be a {\em weak} robust learner with fixed parameters $(\epsilon_0, \delta_0)=(1/3,1/3)$ for some \emph{unknown} target concept $c$ \wrt~$\calU$. Let $D$ be some unknown distribution over $\calX$ such that $\Prob_{x\sim D}\insquare{\exists z\in \calU(x): c(z)\neq c(x)}=0$. By \prettyref{def:strongly-robust}, with fixed sample complexity $m_0=m_{\bbB}(1/3,1/3)$, for any distribution $\Tilde{D}$ over $\calX$ such that $\Prob_{x\sim \Tilde{D}}\insquare{\exists z\in \calU(x): c(z)\neq c(x)}=0$, with probability at least $1/3$ over $S\sim \Tilde{D}_c^{m_0}$, $\Risk_{\calU}(\bbB(S);\Tilde{D}_c)\leq 1/3$. 

We will now boost the confidence and robust error guarantee of the {\em weak} robust learner $\bbB$ by running boosting with respect to the {\em robust} loss (rather than the standard $0$-$1$ loss). Specifically, fix $(\epsilon,\delta)\in (0,1)$ and a sample size $m(\epsilon,\delta)$ that will be determined later. Let $S=\SET{(x_1,y_1),\dots,(x_m,y_m)}$ be an i.i.d. sample from $D_c$. Run the $\alpha$-Boost algorithm on dataset $S$ using $\bbB$ as the weak robust learner for a number of rounds $T$ that will be determined below. On each round $t$, $\alpha$-Boost computes an empirical distribution $D_t$ over $S$ by applying the following update for each $(x,y)\in S$:
\[D_{t}(\SET{(x,y)}) = \frac{D_{t-1}(\SET{(x,y)})}{Z_{t-1}} \times \begin{cases} 
                                                              e^{-2\alpha},& \text{if }\ind[\forall z\in \calU(x): h_{t-1}(z) = y]=1;\\
                                                              1, &\text{otherwise,}
                                                           \end{cases}
\]
where $Z_{t-1}$ is a normalization factor, $\alpha$ is a parameter that will be determined below, and $h_{t-1}$ is the {\em weak} robust predictor outputted by $\bbB$ on round $t-1$ that satisfies $\Risk_\calU(h_{t-1};D_{t-1})\leq 1/3$. Once $D_t$ is computed, we sample $m_0$ examples from $D_t$ and run {\em weak} robust learner $\bbB$ on these examples to produce a hypothesis $h_t$ with robust error guarantee $\Risk_\calU(h_{t};D_{t})\leq 1/3$. This step has failure probability at most $\delta_0=1/3$. We will repeat it for at most $\ceil{\log(2T/\delta)}$ times, until $\bbB$ succeeds in finding $h_t$ with robust error guarantee $\Risk_\calU(h_{t};D_{t})\leq 1/3$. By a union bound argument, we are guaranteed that with probability at least $1-\delta/2$, for each $1\leq t\leq T$, $\Risk_\calU(h_{t};D_{t})\leq 1/3$. Following the argument from \citet*[][Section 6.4.2]{schapire:12}, after $T$ rounds we are guaranteed 
\begin{equation*}
\min_{(x,y) \in S} \frac{1}{T} \sum_{t=1}^{T} \ind[ \forall z \in \calU(x): h_{t}(z)=y ]
\geq \frac{2}{3} - \frac{2}{3}\alpha - \frac{\ln(|S|)}{2\alpha T},
\end{equation*}
so we will plan on running until round 
$T = 1 + 48 \ln(|S|)$ 
with value 
$\alpha = 1/8$ 
to guarantee
\begin{equation*}
\min_{(x,y) \in S} \frac{1}{T} \sum_{t=1}^{T} \ind[ \forall z \in \calU(x): h_{t}(z)=y ]
> \frac{1}{2},
\end{equation*}
so that the majority-vote classifier $\MAJ(h_1,\dots,h_T)$ achieves \emph{zero} robust loss on the empirical dataset $S$, $\Risk_{\calU}(\MAJ(h_1,\dots,h_L);S)=0$. 

Note that each of these classifiers $h_{t}$ is equal to $\bbB(S'_{t})$ for some $S'_{t} \subseteq S$ with $|S'_{t}|=m_0$. Thus, the classifier $\MAJ(h_1,\dots,h_T)$ is representable as the value of an (order-dependent) reconstruction function $\phi$ with 
a compression set size $m_0T=m_0O(\log m)$. Now, invoking \prettyref{lem:robust-compression}, with probability at least $1-\delta/2$,
\[\Risk_{\calU}(\MAJ(h_1,\dots,h_T);\calD)\leq O\inparen{\frac{m_0\log^2 m}{m} + \frac{\log(2/\delta)}{m}},\]
and setting this less than $\epsilon$ and solving for a sufficient size of $m$ yields the stated sample complexity bound. 
\end{proof}

\section{Robustness at Different Levels of Granularity}
\label{app:granular}

For concreteness, throughout the rest of this section, we consider robustness with respect to metric balls $\B_\gamma(x)=\SET{z\in\calX: \rho(x,z)\leq \gamma}$ where $\rho$ is some metric on $\calX$ (e.g., $\ell_\infty$ metric), and $\gamma > 0$ is the perturbation radius. Achieving small robust risk \wrt a \emph{fixed} perturbation set $\B_\gamma$ is the common goal studied in adversarially robust learning. What we studied so far in this work is learning a predictor $\hat{h}$ robust to $\calU=\B_\gamma$ perturbations as measured by the robust risk: $\Prob_{(x,y)\sim D}\insquare{\exists z\in \B_\gamma(x): \hat{h}(z)\neq y}$, when given access to a learner $\bbA$ barely robust \wrt $\calU^{-1}(\calU)=\B_{2\gamma}$.

Our original approach to boosting robustness naturally leads us to an alternate interesting idea: learning a cascade of robust predictors with different levels of granularity. This might be desirable in situations where it is difficult to robustly learn a distribution $D_c$ over $\calX\times \calY$ with robustness granularity $\gamma$ everywhere, and thus, we settle for a weaker goal which is first learning a robust predictor $h_1$ with granularity $\gamma$ on say $\beta$ mass of $D$, and then recursing on the conditional distribution of $D$ where $h_1$ is not $\gamma$-robust and learning a robust predictor $h_2$ with granularity $\gamma/2$, and so on. That is, we are adaptively learning a sequence of predictors $h_1,\dots, h_T$ where each predictor $h_t$ is robust with granularity $\frac{\gamma}{2^{t}}$. Furthermore, if we are guaranteed that in each round we make progress on some $\beta$ mass then it follows that 
\[ \Prob\insquare{\cup_{t=1}^{T} \Rob_{\gamma/2^{t-1}}(h_t)} = 1 - \Prob\insquare{\cap_{t=1}^{T} \overline{\Rob}_{\gamma/2^{t-1}}(h_t)} \geq 1 - (1-\beta)^T,\]
and the cascade predictor $\CAS(h_{1:T})$ has the following robust risk guarantee
\begin{equation*}
\sum_{t=1}^{T} \Prob\hspace{-0.05cm}\insquare{\bar{R}_{1:t-1}\hspace{-0.05cm}\wedge\hspace{-0.05cm} \inparen{\exists z\in \B_{\gamma/2^{t}}(x)\hspace{-0.05cm}: \forall_{t'< t}G_{h_{t'}}(z)=\perp \wedge G_{h_{t}}(z)\hspace{-0.05cm}=\hspace{-0.05cm}1\hspace{-0.05cm}-\hspace{-0.05cm}c(x)} } + \Prob\hspace{-0.05cm}\insquare{\bar{R}_{1:T}} \leq \frac{\epsilon}{\beta} + (1-\beta)^T.
\end{equation*}

In words, the cascade predictor $\CAS(h_{1:T})$ offers robustness at different granularities. That is, for $x\sim D$ such that $x\in R_1$, $\CAS(h_{1:T})$ is guaranteed to be robust on $x$ with granularity $\gamma$, and for $x\sim D$ such that $x\in \bar{R}_1 \cap R_2$, $\CAS(h_{1:T})$ is guaranteed to be robust on $x$ with granularity $\gamma/2$, and so on. 

\paragraph{Applications.} We give a few examples where this can be useful. Consider using SVMs as barely robust learners. SVMs are known to be margin maximizing learning algorithms, which is equivalent to learning linear predictors robust to $\ell_2$ perturbations. In our context, by combining SVMs with our boosting algorithm, we can learn a cascade of linear predictors each with a maximal margin on the conditional distribution.
\end{document}